\documentclass{article}

\usepackage{microtype}
\usepackage{graphicx}
\usepackage{subfig}
\usepackage{booktabs} 

\usepackage{hyperref}

\usepackage[accepted]{icml2019} 

\icmltitlerunning{Analogies Explained: Towards Understanding Word Embeddings}

\usepackage[utf8]{inputenc} 
\usepackage[T1]{fontenc}    
\usepackage{hyperref}       
\usepackage{url}            
\usepackage{booktabs}       
\usepackage{amsfonts}       
\usepackage{amssymb, amsmath, amsthm}
\usepackage{thmtools,thm-restate}
\usepackage{bm}
\usepackage{nicefrac}       
\usepackage{microtype}      
\graphicspath{{figures/}}
\usepackage{tikz}           
\usetikzlibrary{matrix}
\usetikzlibrary{automata,positioning}
\usetikzlibrary{arrows}
\usepackage{multirow}
\usepackage{xfrac}
\usepackage{relsize}
\usepackage[title]{appendix}
\usepackage{enumitem}       
\usepackage[hang,flushmargin]{footmisc}

\newcommand{\keypoint}[1]{\vspace{0.1cm}\noindent\textbf{#1}\quad}

\newcommand{\sm}{\scalebox{0.7}[1.0]{\( - \)}}

\newcommand{\vv}[1]{\mathbf{w}_{#1}}
\newcommand{\vtv}[2]{{\mathbf{v}_{#1}}^\top\mathbf{v}_{#2}'}
\newcommand{\RR}[1]{\mathbb{R}^{#1}}
\newcommand{\winW}{w_i\!\in\!\mathcal{W}}
\newcommand{\bmu}[1]{\bm{#1}}
\newcommand{\msc}[1]{\mathsmaller{\mathcal{#1}}}

\newtheorem{theorem}{Theorem}
\newtheorem{corollary}{Corollary}[theorem]
\newtheorem{lemma}{Lemma} 
 
\newtheorem{assumption}{A\!} 
\newtheorem{definition}{Definition D\!}

\DeclareMathOperator*{\argmin}{argmin}
\DeclareMathOperator*{\argmax}{argmax}

\begin{document}

\twocolumn[
\icmltitle{Analogies Explained: Towards Understanding Word Embeddings}



\icmlsetsymbol{equal}{*}

\begin{icmlauthorlist}
\icmlauthor{Carl Allen}{ed}
\icmlauthor{Timothy Hospedales}{ed}
\end{icmlauthorlist}

\icmlaffiliation{ed}{School of Informatics, University of Edinburgh}

\icmlcorrespondingauthor{Carl Allen}{carl.allen@ed.ac.uk}

\icmlkeywords{Machine Learning, Embedding, NLU, Representation}

\vskip 0.3in
]



\printAffiliationsAndNotice{}  

\begin{abstract}
    Word embeddings generated by neural network methods such as \textit{word2vec} (W2V) are well known to exhibit seemingly linear behaviour, e.g. the embeddings of analogy \emph{``woman is to queen as man is to king''} approximately describe a parallelogram. This property is particularly intriguing since the embeddings are not trained to achieve it. Several explanations have been proposed, but each introduces assumptions that do not hold in practice.
    We derive a probabilistically grounded definition of \textit{paraphrasing} that we re-interpret as \textit{word transformation}, a mathematical description of \emph{``$w_x$ is to $w_y$''}. From these concepts we prove existence of linear relationships between W2V-type embeddings that underlie the analogical phenomenon, identifying explicit error terms.
\end{abstract}

\section{Introduction}

The vector representation, or \textit{embedding}, of words underpins much of modern machine learning for natural language processing (e.g. \citet{turney2010frequency}). Where, previously, embeddings were generated explicitly from word statistics, neural network methods are now commonly used to generate \textit{neural embeddings} that are of low dimension relative to the number of words represented, yet achieve impressive performance on downstream tasks (e.g. \citet{turian2010word, socher2013parsing}). Of these, \textit{word2vec}\footnote[2]{Throughout, we refer to the more commonly used \textit{Skipgram} implementation of W2V with negative sampling (SGNS).} (W2V) \citep{mikolov2013distributed} and \textit{Glove} \cite{pennington2014glove} are amongst the best known and on which we focus.

Interestingly, such embeddings exhibit seemingly linear behaviour \cite{mikolov2013linguistic, levy2014linguistic}, e.g. the respective embeddings of \textit{analogies}, or word relationships of the form \emph{``$w_a$ is to $w_{a^*}$ as $w_b$ is to $w_{b^*\!}$''}, often satisfy $\mathbf{w}_{a^*}-\mathbf{w}_{a}+\mathbf{w}_{b}\approx\mathbf{w}_{b^*}$, where $\mathbf{w}_{i}$ is the embedding of word $w_i$. This enables analogical questions such as ``\emph{man is to king as woman is to ..?}'' to be solved by vector addition and subtraction. Such high order structure is surprising since word embeddings are trained using only pairwise word co-occurrence data extracted from a text corpus.

We first show that where embeddings factorise \textit{pointwise mutual information} (PMI), it is \textit{paraphrasing} that determines when a linear combination of embeddings equates to that of another word. We say $king$ paraphrases $man$ and $royal$, for example, if there is a semantic equivalence between $king$ and $\{man, royal\}$ combined. We can measure such equivalence with respect to probability distributions over nearby words, in line with Firth's maxim ``\emph{You shall know a word by the company it keeps}'' \cite{firth1957synopsis}. 
We then show that paraphrasing can be reinterpreted as \textit{word transformation} with additive \textit{parameters} (e.g. from $man$ to $king$ by adding $royal$) and generalise to also allow subtraction. 
Finally, we prove that by interpreting an analogy \emph{``$w_a$ is to $w_{a^*}$ as $w_b$ is to $w_{b^*}$''} as word transformations $w_a$ to $w_{a^*}$ and $w_b$ to $w_{b^*}$ sharing the same parameters, the linear relationship observed between word embeddings of analogies follows (see overview in Fig \ref{fig:summary}). 
Our key contributions are:
\vspace{-0.30cm}
\begin{itemize}[leftmargin=.5cm]
    \item to derive a probabilistic definition of \textit{paraphrasing} and show that it governs the relationship between one (PMI-derived) word embedding and any sum of others;
\vspace{-0.25cm}
    \item to show how paraphrasing can be generalised and interpreted as the \textit{transformation} from one word to another, giving a mathematical formulation for \emph{``$w_x$ is to $w_{x^*\!}$''};
\vspace{-0.25cm}
    \item to provide the first rigorous proof of the linear relationship between word embeddings of analogies, including explicit, interpretable error terms; and
\vspace{-0.25cm}
    \item to show how these relationships materialise between vectors of PMI values, and so too in word embeddings that factorise the PMI matrix, or approximate such a factorisation e.g. W2V and \textit{Glove}.
\end{itemize}
\begin{figure}[t!]
    \centering
    \includegraphics[width=\linewidth, trim=1.25cm 2.0cm 1cm 2.9cm, clip]{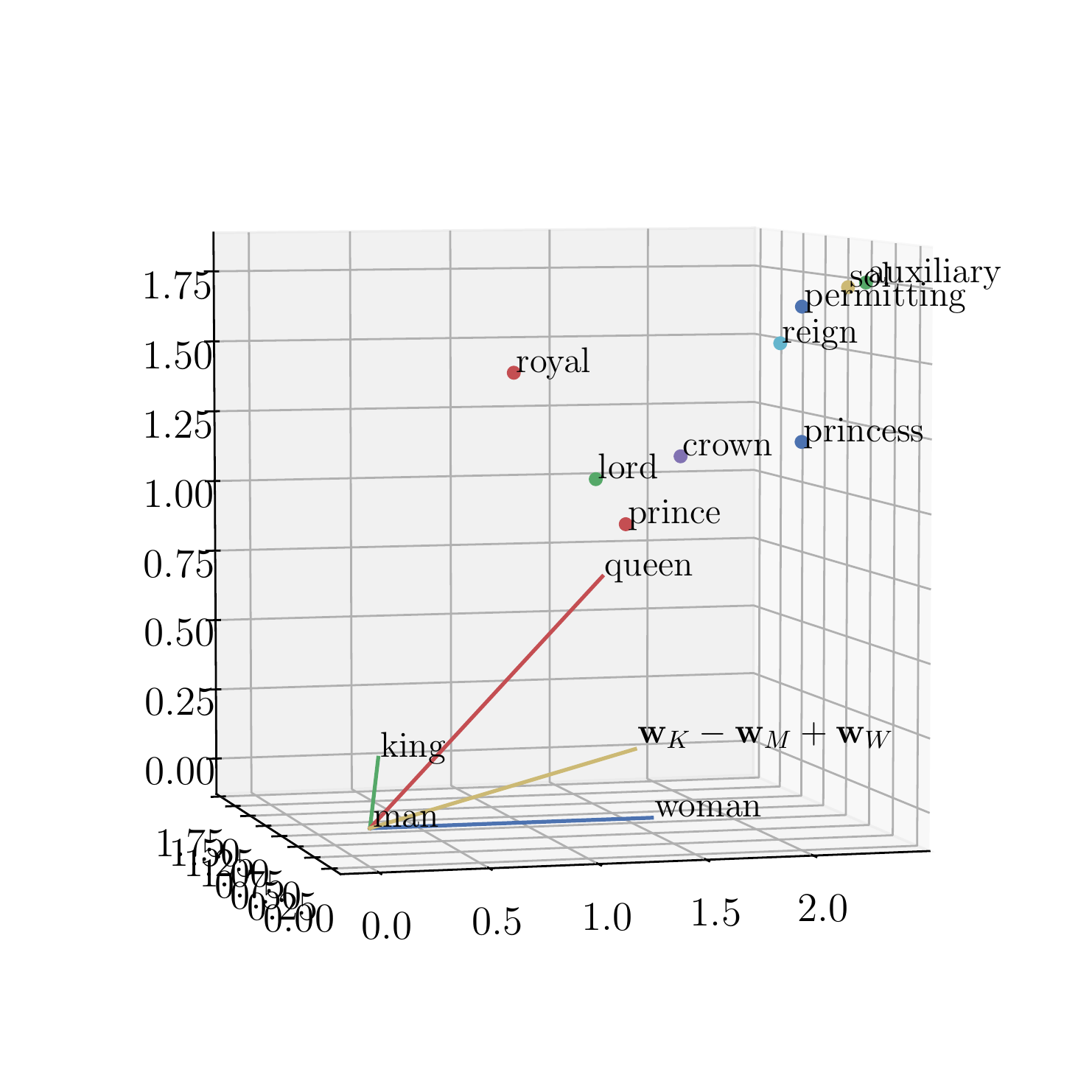}
    \vspace{-10pt}
    \caption{The relative locations of word embeddings for the analogy "\textit{man} is to \textit{king} as \textit{woman} is to ..?". The closest embedding to the linear combination  $\mathbf{w}_{K}-\mathbf{w}_{M}+\mathbf{w}_{W}$ is that of \textit{queen}. We explain why this occurs and interpret the difference between them.}
    \label{fig:the_gap}
\vspace{-5pt}
\end{figure}

\section{Previous Work}
%
Intuition for the presence of linear analogical relationships, or \textit{linguistic regularity}, amongst word embeddings was first suggested by \citet{mikolov2013distributed, mikolov2013linguistic} and \citet{pennington2014glove}, and has been widely discussed since (e.g. \citet{levy2014linguistic, linzen2016issues}).
More recently, several theoretical explanations have been proposed:
\begin{itemize}[leftmargin=.5cm]
	\item \citet{arora2016latent} propose a latent variable model for language that contains several strong \textit{a priori} assumptions about the spatial distribution of word vectors, discussed by \citet{gittens2017skip}, that we do not require. 
    Also, the two embedding matrices of W2V are assumed equal, which we show to be false in practice.
\vspace{-0.2cm}
    \item \citet{gittens2017skip} refer to \textit{paraphrasing}, from which we draw inspiration, but make several assumptions that fail in practice: (i) that words follow a uniform distribution rather than the (highly non-uniform) Zipf distribution; (ii) that W2V learns a conditional distribution -- violated by negative sampling \cite{levy2014neural}; and (iii) that joint probabilities beyond pairwise co-occurrences are zero. 
\vspace{-0.2cm}
    \item \citet{ethayarajh2018towards} offer a recent explanation based on \textit{co-occurrence shifted PMI}, however that property lacks motivation and several assumptions fail, e.g.
    it requires more than for opposite sides to have equal length to define a parallelogram in $\mathbb{R}^d\!,\ d>2$ (their Lemma 1).
\end{itemize}
\vspace{-0.25cm}
To our knowledge, no previous work mathematically interprets analogies so as to rigorously explain why if \emph{``$w_a$ is to $w_{a^*}$ as $w_b$ is to $w_{b^*\!}$''} then a linear relationship manifests between correponding word embeddings.

\section{Background}\label{sec:bg}

%
The \textbf{Word2Vec} algorithm considers a set of word pairs $\smash{\{\smash{(w_{i_k}, c_{j_k})}\}_k}$ generated from a (typically large) text corpus, by allowing the \textit{target} word $w_i$ to range over the corpus, and the \textit{context} word $c_j$ to range over a context window (of size $l$) symmetric about the target word. For each observed word pair (\textit{positive sample}), $k$ random word pairs (\textit{negative samples}) are generated according to monogram distributions. 
The 2-layer ``neural network'' architecture simply multiplies two weight matrices $\mathbf{W}, \mathbf{C}\!\in\!\mathbb{R}^{d\times n}$, subject to a non-linear (sigmoid) function, where $d$ is the embedding dimensionality and $n$ is the size of $\mathcal{E}$ the dictionary of unique words in the corpus.
Conventionally, $\mathbf{W}$ denotes the matrix closest to the input target words.
Columns of $\mathbf{W}$ and $\mathbf{C}$ are the \textit{embeddings} of words in $\mathcal{E}$: $\mathbf{w}_{i}\!\in\!\mathbb{R}^{d}$ ($i^\text{th}$ column of $\mathbf{W}$) corresponds to $w_i$ the $i^{th}$ word in $\mathcal{E}$ observed as a target word; and $\mathbf{c}_{i}\!\in\!\mathbb{R}^{d}$ ($i^\text{th}$ column of $\mathbf{C}$) corresponds to $\smash{c_i}$, the same word when observed as a context word.

\citet{levy2014neural} identified that the objective function for W2V is optimised if:
\vspace{-5pt}
\begin{equation}\label{eq:w2v_pmi}
    \mathbf{w}_{i}^\top\mathbf{c}_{j} 
    \  = \
    \text{PMI}(w_i,c_j) - \log{k}\ ,
    \vspace{-5pt}
\end{equation}
where $\text{PMI}(w_i,c_j)\!=\!\log{\tfrac{p(w_i,\,c_j)}{p(w_i)p(c_j)}}$ is known as \textit{pointwise mutual information}.
In matrix form, this equates to:
\vspace{-5pt}
\begin{equation}\label{eq:w2v_pmi_mx}
    \mathbf{W}^\top\mathbf{C} 
    \   = \ 
    \mathbf{SPMI}\ \in\mathbb{R}^{n\times n}\ ,
    \vspace{-5pt}
\end{equation}
\noindent where $\mathbf{SPMI}_{i,j}\!=\!\textup{PMI}(w_i,c_j)\!-\! \log{k}$, (\textit{shifted} PMI). 

%
\textbf{\textit{Glove}} \citep{pennington2014glove} has the same architecture as W2V. Its embeddings perform comparably and also exhibit linear analogical structure. \textit{Glove}'s loss function is optimised when:
\vspace{-5pt}
\begin{equation}\label{eq:glove}
    \mathbf{w}_{i}^\top \mathbf{c}_{j}=\log p(w_i, c_j) - b_i - b_j + \log Z   
    \vspace{-5pt}
\end{equation}
for biases $b_i$, $b_j$ and normalising constant $Z$.
(\ref{eq:glove}) generalises (\ref{eq:w2v_pmi}) due to the biases, giving \textit{Glove} greater flexibility than W2V and a potentially wider range of solutions. However, we will show that it is factorisation of the PMI matrix that causes linear analogical structure in embeddings, as approximately achieved by W2V (\ref{eq:w2v_pmi}). We conjecture that the same rationale underpins analogical structure in \textit{Glove} embeddings, perhaps more weakly due to its increased flexibility.

\section{Preliminaries}\label{sec:P2V}
We consider pertinent aspects of the relationship between word embeddings and co-occurrence statistics (\ref{eq:w2v_pmi}, \ref{eq:w2v_pmi_mx}) relevant to the
linear structure between embeddings of analogies:

\keypoint{Impact of the \textit{Shift}}
As a chosen hyper-parameter, reflecting nothing of word properties, any effect on embeddings of $k$ appearing in (\ref{eq:w2v_pmi}) is arbitrary. Comparing typical values of $k$ with empirical PMI values (Fig \ref{fig:pmi_hist}), shows that the so-called \textit{shift} $(\sm\log k)$ may also be material.
\begin{figure}[t]
  \centering
    \begin{minipage}{0.4\textwidth}
        \centering
        \includegraphics[width=1\textwidth]{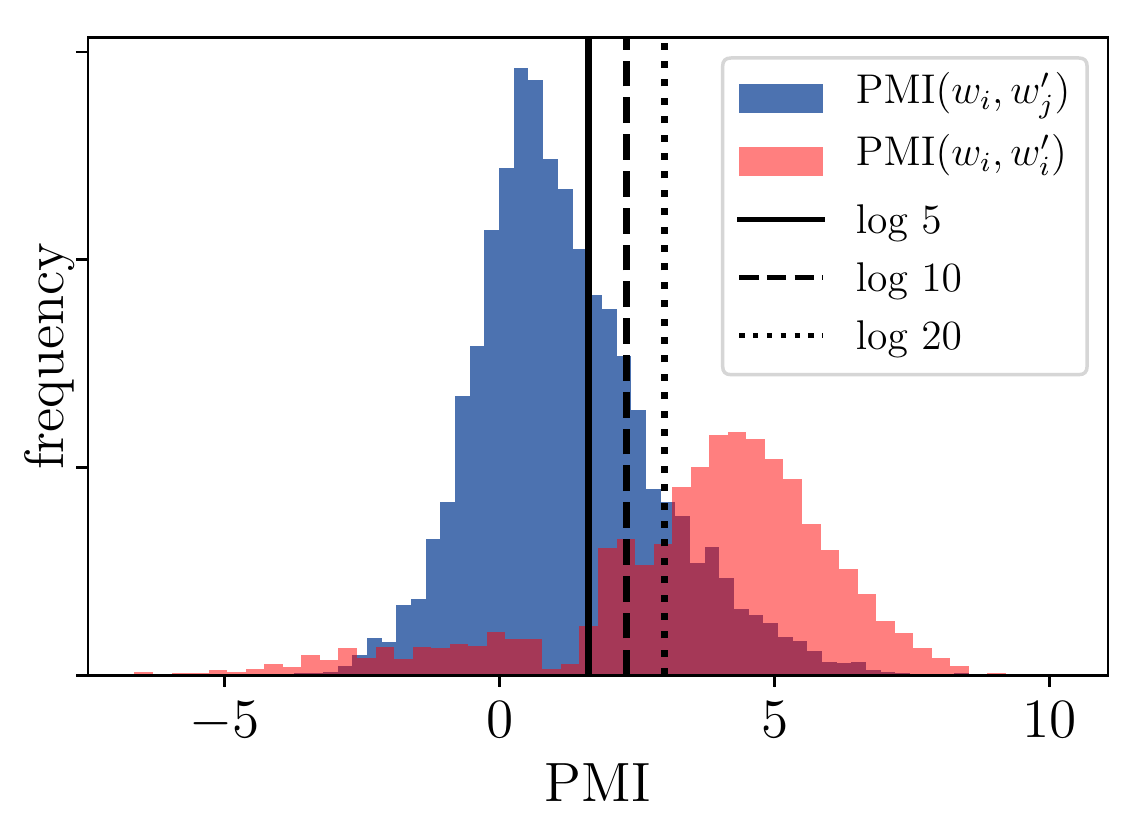}
    \end{minipage}
    \vspace{-10pt}
    \captionsetup{justification=centering,margin=0.01cm}
    \caption{Histogram of $\text{PMI}(w_i, c_j)$ for word pairs randomly sampled from text (blue) with $\text{PMI}(w_i, c_i)$ for the \textit{same word} overlaid (red, scale enlarged). The \textit{shift} is material for typical values of $k$.}
    \label{fig:pmi_hist}
    \vspace{-5pt}
\end{figure}
Further, it is observed that adjusting the W2V algorithm to avoid any direct impact of the \textit{shift} improves embedding performance \cite{minh}. 
We conclude that the \textit{shift} is a detrimental artefact of the W2V algorithm and, unless stated otherwise, consider embeddings that factorise the \textit{unshifted} PMI matrix:
\vspace{-0.25cm}
\begin{align}\label{eq:w2v_clean}
    \mathbf{w}_i^\top\mathbf{c}_j
     \  = \ 
    \text{PMI}(w_i,c_j)		
    \quad\text{or}\quad
    \mathbf{W}^\top\mathbf{C} 
     \  = \ 
    \mathbf{PMI}\ .
\end{align}
\keypoint{Reconstruction Error}
In practice, (\ref{eq:w2v_pmi_mx}) and (\ref{eq:w2v_clean}) hold only \textit{approximately} since $\mathbf{W}^\top\mathbf{C}\in\RR{n\!\times\!n}$ is rank-constrained (rank $r\!\ll\! d\!<\!n$)  relative to the factored matrix $\mathbf{M}$, e.g. $\mathbf{M}\!=\!\mathbf{PMI}$ in (\ref{eq:w2v_clean}). Recovering elements of $\mathbf{M}$ from $\mathbf{W}$ and $\mathbf{C}$ is thus subject to \textit{reconstruction error}. However, we rely throughout on linear relationships in $\RR{n}$, requiring only that they are sufficiently maintained when projected ``down'' into $\RR{d}$, the space of embeddings. To ensure this, we assume:
\begin{assumption}\label{ass:full_rank}
	$\mathbf{C}$ has full row rank.
\end{assumption}
\begin{assumption}\label{ass:pseudo-inv}
Letting $\mathbf{M}_k$ denote the $k^{th}$ column of factored matrix $\mathbf{M}\!\in\!\RR{n\times n}$, the projection $f\!\!:\!\!\RR{n}\!\rightarrow\!\RR{d}$, $f(\mathbf{M}_{i})\!=\!\mathbf{w}_{i}$ is approximately homomorphic with respect to addition, i.e. $f(\mathbf{M}_{i}+\mathbf{M}_{j}) \approx f(\mathbf{M}_{i})+f(\mathbf{M}_{j})$.
\end{assumption}
\vspace{-0.20cm}
\noindent
A\ref{ass:full_rank} is reasonable since $d\!\ll\! n$ and $d$ is chosen. 
A\ref{ass:pseudo-inv} means that, whatever the factorisation method used (e.g. analytic, W2V, \textit{Glove}, weighted matrix factorisation \cite{srebro2003weighted}), linear relationships between columns of $\mathbf{M}$ are sufficiently preserved by columns of $\mathbf{W}$, i.e. the embeddings $\mathbf{w}_{i}$. 
For example, minimising a least squares loss function gives the linear projection $\mathbf{w}_{i}\!=\!f_{LSQ}(\mathbf{M}_i)\!=\!\mathbf{C}^\dagger\mathbf{M}_i$ for which A\ref{ass:pseudo-inv} holds exactly (where $\mathbf{C}^\dagger\!=\!(\mathbf{CC}^\top)^{-1}\mathbf{C}$, the \textit{Moore-Penrose pseudo-inverse} of $\mathbf{C}^\top\!$, which exists by A\ref{ass:full_rank});\footnote{
\textit{w.l.o.g.} we write $f(\cdot)\!=\!\mathbf{C}^\dagger(\cdot)$ throughout (except in specific cases) to emphasise linearity of the relationship.}
whereas for W2V, $\mathbf{w}_{i}\!=\!f_{W2V}(\mathbf{M}_i)$ is non-linear.\footnote{
It is beyond the scope of this work to show A\ref{ass:pseudo-inv} is satisfied when the W2V loss function is minimised (\ref{eq:w2v_clean}). We instead prove existence of linear relationships in the full rank space of PMI columns,  thus in linear projections thereof, and assume A\ref{ass:pseudo-inv} holds sufficiently for W2V embeddings given (\ref{eq:w2v_pmi_mx}) and empirical observation of linearity.}

\keypoint{Zero Co-occurrence Counts}
The co-occurrence of rare words are often unobserved, thus their empirical probability estimates zero and PMI estimates undefined. 
However, for a fixed dictionary $\mathcal{E}$, such zero counts decline as the corpus or context window size increase (the latter can be arbitrarily large if more distant words are down-weighted, e.g. \citet{pennington2014glove}). Here, we consider small word sets $\mathcal{W}$ and assume the corpus and context window to be of sufficient size that the \textit{true} values of considered probabilities are non-zero and their PMI values well-defined, i.e.:
%
\begin{assumption}\label{ass:all_log_probs}
$p(\mathcal{W})\!>\!0,\ \ \forall\mathcal{W}\!\subseteq\!\mathcal{E}, |\mathcal{W}|\!<\!l, $
\end{assumption}
\vspace{-0.20cm}
where (throughout) ``$|\mathcal{W}|\!<\!l$'' means $|\mathcal{W}|$  \textit{sufficiently less} than $l$.

\keypoint{The Relationship between \textbf{W} and \textbf{C}}
Several works (e.g. \citet{hashimoto2016word, arora2016latent}) assume embedding matrices $\mathbf{W}$ and $\mathbf{C}$ to be equal, i.e. $\mathbf{w}_i\!=\!\mathbf{c}_i\ \,\forall i$. 
The assumption is convenient as the number of parameters is halved, equations simplify and consideration of how to use $\mathbf{w}_i$ and $\mathbf{c}_i$ falls away. 
However, this implies $\mathbf{W}^\top\mathbf{W}=\mathbf{PMI}$, requiring $\mathbf{PMI}$ to be positive semi-definite,
which is not true for typical corpora.
Thus $\mathbf{w}_i$, $\mathbf{c}_i$ are not equal and modifying W2V to enforce them to be would unnecessarily constrain and may well worsen the low-rank approximation.

\section{Paraphrases}\label{sec:paraphrases}

Following a similar approach to \citet{gittens2017skip}, we consider a small set of target words $\mathcal{W}\!=\!\{w_1, \ldots, w_m\}\!\subseteq\!\mathcal{E}$, $|\mathcal{W}|\!< l$; and the sum of their embeddings $\vv{\msc{W}}\!=\!\sum_i\mathbf{w}_{i}$. 
In practice, we say word $w_*\!\in\!\mathcal{E}$ \textit{paraphrases} $\mathcal{W}$ if $w_*$ and $\mathcal{W}$ are semantically interchangeable within the text, i.e. in circumstances where \textit{all} $\winW$ appear, $w_*$ could appear instead. This suggests a relationship between the probability distributions $p(c_j|\mathcal{W})$ and $p(c_j|w_*)$, $\forall c_j\!\in\!\mathcal{E}$. We refer to such conditional distributions over all context words as the \textit{distribution} \textit{induced} by $\mathcal{W}$ or $w_*$, respectively. 
\vspace{-5pt}

\subsection{Defining a Paraphrase}\label{sec:paraphrase_ID}

Let $\mathcal{C}_\msc{W}\!=\!\{c_{j_1}, \ldots, c_{j_t}\}$ be a sequence of words (with repetition) observed in the context of $\mathcal{W}$.\footnote{By symmetry, $\mathcal{C}_\msc{W}$ is the set of target words for which all $\winW$ are simultaneously observed in the context window.} A \textit{paraphrase word} $w_*\!\in\!\mathcal{E}$ can be thought of as that which \textit{best explains} the observation of  $\mathcal{C}_\msc{W}$. 
From a maximum likelihood perspective we have $w_*^{\mathsmaller{(1)}} \!=\!\argmax_{w_i\in\mathcal{E}} p(\mathcal{C}_\msc{W}|w_i)$. 
Assuming $c_j\!\in\!\mathcal{C}_\msc{W}$ to be independent draws from $p(c_j|\mathcal{W})$, gives:
\vspace{-2pt}
\begin{align*}
    w_*^{\mathsmaller{(1)}} 
     \ = \ &
    \argmax_{w_i}\,\mathsmaller{\prod_{c_j\in\mathcal{E}}}\,p(c_j|w_i)^{\#_j}
    \nonumber\\
     \ \to \, &
     \argmax_{w_i}\,\mathsmaller{\sum_{c_j\in\mathcal{E}}}\,p(c_j|\mathcal{W})\log p(c_j|w_i)\ ,
\end{align*}
as $|\,\mathcal{C}_\msc{W}|\!\to\!\infty$, where $\#_j$ denotes the count of $c_j$ in $\mathcal{C}_\msc{W}$. 
It follows that $w_*^{\mathsmaller{(1)}}$ minimises the Kullback-Leibler (KL) divergence $\Delta_{\mathsmaller{KL}}^{\msc{W}, w_*}$ between the induced distributions, i.e.:
\vspace{-2pt}
\begin{align*}
    \Delta_{\mathsmaller{KL}}^{\msc{W}, w_*}
     & \ = \ 
    D_{\mathsmaller{KL}}[\,P(c_j| \mathcal{W})\,||\,P(c_j| w_*)\,]\\
     & \ = \ 
    \mathsmaller{\sum_j} p(c_j| \mathcal{W})
       \log \tfrac{p(c_j| \mathcal{W})}{ p(c_j| w_*)} \ .
\end{align*}
Alternatively, we might consider $w_*^{\mathsmaller{(2)}}\!$, the target word whose set of associated context words 
$\mathcal{C}_{w_*}$ is best explained by $\mathcal{W}$, in the sense that $w_*^{\mathsmaller{(2)}}\!$ minimises KL divergence  $\Delta_{\mathsmaller{KL}}^{\mathsmaller{w_*, \mathcal{W}}} \!=\! D_{\mathsmaller{KL}}[P(c_j| w_*)\,||\,P(c_j| \mathcal{W})]$ (where, in general, $\Delta_{\mathsmaller{KL}}^{\mathsmaller{\mathcal{W}, w_*}} \!\neq\! \Delta_{\mathsmaller{KL}}^{\mathsmaller{w_*, \mathcal{W}}}$). 
Interpretations of $w_*^{\mathsmaller{(1)}}$ and $w_*^{\mathsmaller{(2)}}$ are discussed in Appendix \ref{app:KL}. In each case, the KL divergence lower bound (zero) is achieved \textit{iff} the induced distributions are equal, providing a theoretical basis for:

\begin{definition}\label{def:paraphrase}
	We say word $w_*\!\in\!\mathcal{E}$ \underline{paraphrases} word set $\mathcal{W}\!\subseteq\!\mathcal{E}$,  $|\mathcal{W}|\!<\! l$, if the \underline{paraphrase error}  $\bmu{\rho}^{\msc{W}, w_*}\!\in\! \RR{n}$ is (element-wise) small, where:
\vspace{-5pt}
\begin{equation*}
    \bmu{\rho}^{\msc{W}, w_*}_j   \ =\ 
    \log \tfrac{p(c_j|w_*)}{p(c_j|\mathcal{W})}\ , c_j\!\in\!\mathcal{E}. 
\end{equation*}
\end{definition}
\vspace{-10pt}
Note that $\mathcal{W}$ and $w_*$ need not appear similarly often for $w_*$ to paraphrase $\mathcal{W}$, only amongst the same context words. We now connect paraphrasing, a semantic relationship, to relationships between word embeddings.

\subsection{Paraphrase = Embedding Sum + Error}
\begin{restatable}{lemma}{primelemma}
\label{lem:paraphrase}
    For any word $w_*\!\in\!\mathcal{E}$ and word set 
    $\mathcal{W}\!\subseteq\!\mathcal{E}$, $|\mathcal{W}|\!<\!l$:
    \vspace{-5pt}
    \begin{equation}\label{eq:lem1}
        \textup{PMI}_{*}
        =
        \sum_{\winW}\textup{PMI}_{i}
         + \bmu{\rho}^{\msc{W}, w_*}
         + \bmu{\sigma}^{\msc{W}}
         - \tau^\msc{W}\bmu{1}
         \ ,
    \vspace{-5pt}
    \end{equation}
where $\textup{PMI}_{\bullet}$ is the column of $\mathbf{PMI}$ corresponding to $w_{\bullet}\!\in\!\mathcal{E}$,
$\bmu{1}\!\in\!\RR{n}$ is a vector of 1s, and \textit{error} terms
    $\bmu{\sigma}^{\msc{W}}_j \!=\! \log\tfrac{p(\mathcal{W}|c_j)}{\prod_i p(w_i|c_j)}$ and
    $\tau^\msc{W}\!=\!\log \tfrac{p(\mathcal{W})}{\prod_i p(w_i)}$.
\end{restatable}
\vspace{-5pt}
\begin{proof}
(See Appendix \ref{app:paraphrase_alt_proof}.) As Lem~\ref{lem:paraphrase} is central to what follows, we sketch its proof: a correspondence is drawn between the product of distributions induced by each $\winW$ (I) and the distribution induced by $w_*$ (II), by comparison to the distribution induced by joint event $\mathcal{W}$ (III), i.e. observing \textit{all} $\winW$ in the context window. I relates to III by the (in)dependence of $\winW$ (i.e. by $\bmu{\sigma}_{j}^\msc{W}$, $\tau^\msc{W}$ ).\footnote{Analogous to a product of marginal probabilities relating to their joint probability subject to independence.} II relates to III by the paraphrase error $\bmu{\rho}^{\msc{W}, w_*}_j$.
\end{proof}
\vspace{-5pt}
Following immediately from Lem~\ref{lem:paraphrase} we have:
\begin{theorem}[Paraphrase]\label{thm:paraphrase}
    For any word $w_*\!\in\!\mathcal{E}$ and word set  $\mathcal{W}\!\subseteq\!\mathcal{E}$, $|\mathcal{W}|\!<\!l$:
    \vspace{-5pt}
    \begin{equation}\label{eq:thm1_matrix2}
        \vv{*}
		=  
        \vv{\msc{W}}
        \, +\,  \mathbf{C}^\dagger(
            \bmu{\rho}^{\msc{W}, w_*}
            \, +\, 
            \bmu{\sigma}^{\msc{W}}
            \, -\, 
            \tau^\msc{W}\bmu{1}
            ) \ ,
    \vspace{-5pt}
    \end{equation}
    where $\vv{\msc{W}}\!=\!\sum_{\winW}\mathbf{w}_{i}$.
\end{theorem}
\vspace{-10pt}
\begin{proof}
	Multiply (\ref{eq:lem1}) by $\mathbf{C}^\dagger$.
\end{proof}
\vspace{-5pt}
Thm~\ref{thm:paraphrase} shows that an embedding (of $w_*$) and a sum of embeddings (of $\mathcal{W}$) differ by the paraphrase error $\bmu{\rho}^{\msc{W}, w_*}$ \underline{between} $w_*$ and $\mathcal{W}$; and $\bmu{\sigma}^{\msc{W}}$, $\tau^\msc{W}$ (collectively \textit{dependence error}) reflecting relationships \underline{within} $\mathcal{W}$ (unrelated to $w_*$):
\vspace{-0.4cm}
\begin{itemize}[leftmargin=.5cm]
    \item $\bmu{\sigma}^{\msc{W}}$ is a vector reflecting conditional dependencies within $\mathcal{W}$ given each $c_j\!\in\!\mathcal{E}$;
    $\bmu{\sigma}_j^\msc{W}\!=\!0$ \textit{iff} all $\winW$ are conditionally independent given each and every $c_j\!\in\!\mathcal{E}$;
    \item $\tau^\msc{W}$ is a scalar measure of mutual independence of $\winW$ (thus constant $\forall c_j\!\in\!\mathcal{E}$); $\tau^\msc{W}\!=\!0$ \textit{iff} $\winW$ are mutually independent.
\end{itemize}
\begin{corollary}\label{cor:indep_words}
    A word set $\mathcal{W}$ has no associated dependence error \emph{iff} $\winW$ are both mutually independent and conditionally independent given each context word $c_j\!\in\!\mathcal{E}$.
\end{corollary}
Thm~\ref{thm:paraphrase}, which holds for all words $w_*$ and word sets $\mathcal{W}$, explains why and when a paraphrase (e.g. of $\{man, royal\}$ by $king$) can be identified by embedding addition ($\vv{man}+\vv{royal}\approx \vv{king}$). The  phenomenon occurs due to a relationship between PMI vectors in $\RR{n}$ that holds for embeddings in $\RR{d}$ under projection by $\mathbf{C}^\dagger$ (by A\ref{ass:full_rank}, A\ref{ass:pseudo-inv}). The vector error $\vv{*}-\vv{\msc{W}}$ depends on both the paraphrase relationship between $w_*$ and $\mathcal{W}$; and statistical dependencies within $\mathcal{W}$. 

\begin{corollary}\label{cor:paraphrase_conc}
    For word $w_*\!\in\!\mathcal{E}$ and word set $\mathcal{W}\!\subseteq\!\mathcal{E}$, $\vv{*}\approx\vv{\msc{W}}$ if $w_*$ paraphrases $\mathcal{W}$ and $\winW$ are materially independent (i.e. net dependence error is small).
\end{corollary}

\subsection{Do Linear Relationships Identify Paraphrases?}\label{sec:FPs}

The converse of Cor \ref{cor:paraphrase_conc} is false: $\vv{*}\!\approx\!\vv{\msc{W}}$ does not imply $w_*$ paraphrases $\mathcal{W}$. Specifically, \textit{false positives} arise if:
(i) paraphrase and dependence error terms are material but happen to cancel, i.e. \textit{total error} $\,\bmu{\rho}^{\msc{W}, w_*} + \bmu{\sigma}^{\msc{W}} - \tau^\msc{W}\bmu{1}\approx\bmu{0}$; 
or
(ii) material components of the total error fall within the high ($n-d$) dimensional null space of $\mathbf{C}^\dagger$ and project to a small vector difference between $\vv{*}$ and $\vv{\msc{W}}$. 
Case (i) can arise in PMI vectors (Lem~\ref{lem:paraphrase}) and thus lower rank embeddings  also (Thm~\ref{thm:paraphrase}), but is highly unlikely in practice due to the high dimensionality ($n$). Case (ii) can arise only in lower rank embeddings (Thm~\ref{thm:paraphrase}) and might be minimised by a good choice of factorisation or projection method.

\subsection{Paraphrasing in \textit{Explicit} Embeddings} 

Lem~\ref{lem:paraphrase} applies to full rank $\mathbf{PMI}$ vectors, without reconstruction error or case (ii) false positives (Sec \ref{sec:FPs}), explaining the linear relationships observed by
\citet{levy2014linguistic}.
\begin{corollary}\label{cor:paraphrase_explicit}
    Thm~\ref{thm:paraphrase} holds for \emph{explicit} word embeddings, i.e. columns of $\mathbf{PMI}$.
\end{corollary}
\vspace{-15pt}
\begin{proof}
    Choose factorisation $\mathbf{W}\!=\!\mathbf{PMI}$, $\mathbf{C}\!=\!\mathbf{I}$ (the identity matrix) in Thm~\ref{thm:paraphrase}.
\end{proof}

\subsection{Paraphrasing in W2V Embeddings}\label{sec:addition_w2v}

Thm~\ref{thm:paraphrase} extends to W2V embeddings by substituting $\vtv{i}{j} = \text{PMI}(w_i,c_j) - \log k$ and $f_{W2V}$:
\begin{corollary}\label{cor:thm1_w2v}
    Under conditions of Thm~\ref{thm:paraphrase}, W2V embeddings satisfy:
    \vspace{-5pt}
    \begin{equation}\label{eq:thm1_w2v}
        \vv{*}
         = 
        \vv{\msc{W}}
         + f_{W2V}\big(
               \bmu{\rho}^{\msc{W}, w_*}
             + \bmu{\sigma}^{\msc{W}}
             - \tau^\msc{W} \bmu{1} 
         +  \log k(|\mathcal{W}| - 1) \bmu{1}\big)\, .
    \end{equation}
\end{corollary}
\noindent Comparing (\ref{eq:thm1_matrix2}) and (\ref{eq:thm1_w2v}) shows that paraphrases correspond to linear relationships in W2V embeddings with an additional error term linear in $|\mathcal{W}|$, and hence with less accuracy if $|\mathcal{W}|\!>\!1$, than for embeddings that factorise $\mathbf{PMI}$.

\section{Analogies}\label{sec:analogy}

An \textit{analogy} is said to hold for words $w_a, w_{a^*}\!, w_b$, $w_{b^*} \!\in\! \mathcal{E}$ if, in some sense, \emph{``$w_a$ is to $w_{a^*}$ as $w_b$ is to $w_{b^*}\!$''}. 
Since in principle the same relationship may extend further (``... as $w_c$ is to $w_{c^*}$'' etc), we characterise a general analogy $\mathfrak{A}$ by a set of ordered word pairs $S_{\mathfrak{A}}\!\subseteq\!\mathcal{E}\!\times\!\mathcal{E}$, where $(w_x, w_{x^*})\!\in\!S_{\mathfrak{A}}$, $w_x, w_{x^*}\!\in\!\mathcal{E}$, \textit{iff} ``$w_x$ is to $w_{x^*}$ as ... [all other analogical pairs]'' under $\mathcal{\mathfrak{A}}$. Our aim is to explain why respective word embeddings often satisfy:
\vspace{-5pt}
\begin{equation}\label{eq:anlgy_expression}
    \mathbf{w}_{b^*}  \ \approx\   \mathbf{w}_{a^*} - \mathbf{w}_{a} + \mathbf{w}_{b}\ ,
    \vspace{-5pt}
\end{equation}
or why in the more general case:
\vspace{-5pt}
\begin{equation}\label{eq:anlgy_expression_gen}
    \vv{x^*}  - \vv{x}\ \approx\ \mathbf{u}_{\mathfrak{A}}\ ,
    \vspace{-5pt}
\end{equation}
$\forall (w_x, w_{x^*})\!\in\!S_{\mathfrak{A}}$ and vector $\mathbf{u}_\mathfrak{A}\!\in\!\RR{n}$ specific to $\mathfrak{A}$.

We split the task of understanding why analogies give rise to Equations \ref{eq:anlgy_expression} and \ref{eq:anlgy_expression_gen} into: 
\textbf{Q1}) understanding conditions under which word embeddings can be added and subtracted to approximate other embeddings; 
\textbf{Q2}) establishing a mathematical interpretation of ``$w_x$ is to $w_{x^*}$''; and 
\textbf{Q3}) drawing a correspondence between those results.
We  show that all of these can be answered with paraphrasing by generalising the notion to word sets.

\subsection{Paraphrasing Word Sets}

\begin{definition}\label{def:paraphrase_extended}
	We say word set $\mathcal{W}_*\!\subseteq\!\mathcal{E}$ \underline{paraphrases} word set $\mathcal{W}\!\subseteq\!\mathcal{E}$,  $|\mathcal{W}|, |\mathcal{W}_*|\!<\! l$, if \underline{paraphrase error}  $\bmu{\rho}^{\msc{W}, \msc{W}_*}\!\in\! \RR{n}$ is (element-wise) small, where:
    \vspace{-5pt}
    \begin{equation*}
        \bmu{\rho}^{\msc{W}, \msc{W}_*}_j   \ =\ 
        \log \tfrac{p(c_j|\mathcal{W}_*)}{p(c_j|\mathcal{W})}\ , c_j\!\in\!\mathcal{E}. 
        \vspace{-5pt}
    \end{equation*}
\end{definition}
\vspace{-5pt}
D\ref{def:paraphrase_extended} generalises D\ref{def:paraphrase} such that the paraphrase term $\mathcal{W}_*$, previously $w_*$, can be more than one word.\footnote{Equivalently, D\ref{def:paraphrase} is a special case of D\ref{def:paraphrase_extended} with $|\mathcal{W}_*|=1$, hence we reuse terms without ambiguity.} 
Analogously to D\ref{def:paraphrase}, word sets paraphrase one another if they induce equivalent distributions over context words.
Note that paraphrasing under D\ref{def:paraphrase_extended} is both reflexive and symmetric (since $|\bmu{\rho}^{\msc{W}, \msc{W}_*}|=|\bmu{\rho}^{\msc{W}_*, \msc{W}}|$), thus ``$\mathcal{W}_*$ paraphrases $\mathcal{W}$'' and ``$\mathcal{W}$ paraphrases $\mathcal{W}_*$'' are equivalent and denoted $\mathcal{W}\!\approx_\textup{P}\!\mathcal{W}_*$.

Analogues of Lem~\ref{lem:paraphrase} and Thm~\ref{thm:paraphrase} follow:
\begin{restatable}{lemma}{primelemmatwo}
\label{lem:paraphrase_extended}
    For any word sets $\mathcal{W}$, $\mathcal{W}_*\!\subseteq\!\mathcal{E}$, $|\mathcal{W}|$, $|\mathcal{W}_*|\!<\!l$:
    \vspace{-5pt}
    \begin{align}\label{eq:lem2}
        \sum_{\winW_*}\!\!\textup{PMI}_{i}
         =  
        \sum_{\winW}\!\textup{PMI}_{i}
             + 
            \bmu{\rho}&^{\msc{W}, \msc{W}_*}
             + 
            \bmu{\sigma}^{\msc{W}}
             - 
            \bmu{\sigma}^{\msc{W}_*}
            \nonumber\\
            & 
            - 
            (\tau^\msc{W}
             -
            \tau^{\msc{W}_*})\bmu{1}\, .
    \end{align}
\end{restatable}
\vspace{-15pt}
\begin{proof}
(See Appendix \ref{app:paraphrase_ext_proof}.)
\end{proof}
\begin{theorem}[Generalised Paraphrase]
\label{thm:paraphrase_extended}
    For any word sets $\mathcal{W}$, $\mathcal{W}_*\!\subseteq\!\mathcal{E}$, $|\mathcal{W}|,|\mathcal{W}_*|\!<\!l$: 
    \vspace{-5pt}
    \begin{equation*}
        \vv{\msc{W}_*}
		=  
        \vv{\msc{W}}
        \, +\,  \mathbf{C}^\dagger(
            \bmu{\rho}^{\msc{W}, \msc{W}_*}
             + 
            \bmu{\sigma}^{\msc{W}}
             - 
            \bmu{\sigma}^{\msc{W}_*}             - 
            (\tau^\msc{W}
             -
            \tau^{\msc{W}_*})\bmu{1}
            ) \ .
    \vspace{-5pt}
    \end{equation*}
\end{theorem}
\vspace{-10pt}
\begin{proof}
	Multiply (\ref{eq:lem2}) by $\mathbf{C}^\dagger$.
\end{proof}
\vspace{-5pt}
Note that $|\mathcal{W}_*|\!=\!1$ recovers Lem~\ref{lem:paraphrase} and Thm~\ref{thm:paraphrase}.
With analogies in mind, we restate Thm~\ref{thm:paraphrase_extended} as:
\begin{restatable}{corollary}{primecorollary}
\label{cor:word_transition}
    For any words $w_x, w_{x^*}\!\in\!\mathcal{E}$ and word sets $\mathcal{W}^+, \mathcal{W}^-\!\subseteq\!\mathcal{E}$, $|\mathcal{W}^+|,|\mathcal{W}^-| < l-1$:
    \vspace{-5pt}
	\begin{align}\label{eq:analogy_from_paraphrase_gen}
        \vv{x^*}
		 =  
        \vv{x}
        + \vv{\msc{W}^+} 
         - \vv{\msc{W}^-}
         + \mathbf{C}^\dagger(
&            \bmu{\rho}^{\msc{W}, \msc{W}_*}
              + 
            \bmu{\sigma}^{\msc{W}}
             - 
            \bmu{\sigma}^{\msc{W}_*}
            \nonumber\\
            &
             - 
            (\tau^\msc{W}
             -
            \tau^{\msc{W}_*})\bmu{1}
            ) ,
	\end{align}
	where $\mathcal{W}\!=\!\{w_x\}\cup\mathcal{W}^+$, $\mathcal{W}_*\!=\!\{w_{x^*}\}\cup\mathcal{W}^-$.
\end{restatable}
\vspace{-10pt}
\begin{proof}
	Set $\mathcal{W}\!=\!\{w_x\}\cup\mathcal{W}^+$, $\mathcal{W}_*\!=\!\{w_{x^*}\}\cup\mathcal{W}^-$ in Thm~\ref{thm:paraphrase_extended}.
\end{proof}
\vspace{-12pt}
Cor \ref{cor:word_transition} shows how any word embedding $\vv{x^*}$ relates to a linear combination of other embeddings ($\vv{\mathsmaller{\Sigma}} \!=\! \vv{x} + \vv{\msc{W}^+}
 - \vv{\msc{W}^-}$), due to an equivalent relationship between columns of $\mathbf{PMI}$. 
Analogously to one-word (D\ref{def:paraphrase}) paraphrases, the vector difference $\vv{x^*}-\vv{\mathsmaller{\Sigma}}$ depends on the paraphrase error that reflects the relationship \underline{between} the two word sets $\mathcal{W}_*$, $\mathcal{W}$; and the dependence error that reflects statistical dependence between words \underline{within} each of $\mathcal{W}$ and $\mathcal{W}_*$. 

\begin{corollary}\label{cor:paraphrase_conc_gen}
For terms as defined above, 
$\vv{x^*}\!\approx\!\vv{x} + \vv{\msc{W}^+}
 - \vv{\msc{W}^-}$ \emph{if} $\mathcal{W}_*\!\approx_\textup{P}\!\mathcal{W}$ 
and $\winW$ and $\winW_*$ are materially independent or dependence terms materially cancel.
\end{corollary}
\vspace{-2pt}
False positives can arise as discussed in Sec 
\ref{sec:FPs}.

\subsection{From Paraphrases to Analogies}

A special case of Cor \ref{cor:word_transition} gives:
\begin{corollary}\label{cor:word_trans_analogies} For any
$w_a, w_{a^*}, w_b, w_{b^*}\!\in\!\mathcal{E}$:
    \vspace{-5pt}
	\begin{align}\label{eq:analogy_from_paraphrase}
        \mathbf{w}_{b^*}
		=\  
        \mathbf{w}_{a^*}
        - \mathbf{w}_{a}
        + \mathbf{w}_{b}
        + \mathbf{C}^\dagger(
            \bmu{\rho}&^{\msc{W}, \msc{W}_*}
              + 
            \bmu{\sigma}^{\msc{W}}
            -  
            \bmu{\sigma}^{\msc{W}_*}
            \nonumber\\
            & 
             - 
            (\tau^\msc{W}
            -
            \tau^{\msc{W}_*})\bmu{1}
            )\, ,
	\end{align}
	where $\mathcal{W}\!=\!\{w_b, w_{a^*}\}$ and
$\mathcal{W}_*\!=\!\{w_{b^*}, w_a\}$.
\end{corollary}
\vspace{-10pt}
\begin{proof}
	Set $w_x\!=\!w_b$, $w_{x^*}\!=\!w_{b^*}$, 
$\mathcal{W}^+\!=\!\{w_{a^*}\}$, $\mathcal{W}^-\!=\!\{w_a\}$ in Cor \ref{cor:word_transition}.
\end{proof}
\vspace{-5pt}
Thus we see that (\ref{eq:anlgy_expression}) holds \textit{if} 
$\{w_{b^*}, w_a\} \!\!\approx_\textup{P}\!\! \{w_b, w_{a^*}\!\}$  
and those word pairs exhibit \textit{similar dependence} (Sec \ref{sec:similar_dep}).
More generally, by Cor \ref{cor:word_transition} we see that (\ref{eq:anlgy_expression_gen}) is satisfied by 
$\mathbf{u}_{\mathfrak{A}} \!\approx\! \vv{\msc{W\!^+}} \sm\, \vv{\msc{W\!^-}}\!$
\textit{if} 
$\{w_{x^*}\!, \mathcal{W}^-\!\}  \!\approx_\textup{P}\!  \{w_x, \mathcal{W}^+\!\}$ 
$\forall (w_x, w_{x^*}\!)\!\in\! S_\mathfrak{A}$ 
for \textit{common} word sets $\mathcal{W}^+\!, \mathcal{W}^-\!\subseteq\!\mathcal{E}$ and each pair of paraphrasing word sets exhibit similar dependence.

This establishes sufficient conditions for the linear relationships observed in analogy embeddings (\ref{eq:anlgy_expression}, \ref{eq:anlgy_expression_gen}) in terms of semantic relationships, answering Q1. 
However, those relationships are \textit{paraphrases}, with no obvious connection to the ``\emph{$w_x$ is to $w_{x^*}...$}'' relationships of analogies. 
We now show that paraphrases sufficient for (\ref{eq:anlgy_expression}, \ref{eq:anlgy_expression_gen}) correspond to analogies by introducing the concept of \textit{word transformation}.

\subsection{Word Transformation} 

The paraphrase of a word set $\mathcal{W}$ by word $w_*$ (D\ref{def:paraphrase}) has, so far, been considered in terms of an equivalence between $\mathcal{W}$ and $w_*$ by reference to their induced distributions. 
Alternatively, that paraphrase can be interpreted as a \textit{transformation} from an arbitrary $w_s\!\in\!\mathcal{W}$ to $w_*$ by adding words $\mathcal{W}^+\!=\!\{\winW, w_i\!\neq\! w_s\}$. 
Notionally, $\mathcal{W}^+$ can be considered ``words that make $w_s$ more like $w_{*}$''. More precisely, $\winW^+$ \textit{add context} to $w_{s}$: we move from a distribution induced by $w_{s}$ alone to one induced by the \textit{joint} event of simultaneously observing $w_s$ and all $\winW^+$, a \textit{contextualised} occurrence of $w_{s}$ with an induced distribution closer that of $w_{*}$. 
A similar view can be taken of the associated embedding addition: starting with $\vv{s}$, add $\mathbf{w}_{i}\ \forall\winW^+$\! to approximate $\vv{*}$. Note that only \textit{addition} applies. 

Moving to D\ref{def:paraphrase_extended}, the paraphrase of one word set $\mathcal{W}$ by another $\mathcal{W}_*$ can be interpreted additively as starting with some $w_{x}\!\in\!\mathcal{W}$, $w_{x^*}\!\in\!\mathcal{W}_*$, and adding $\mathcal{W}^+\!\!=\!\{\winW, w_i\!\neq\! w_x\}$, $\mathcal{W}^-\!\!=\!\{\winW_*, w_i\!\neq\! w_{x^*}\!\}$, respectively, such that the resulting sets $\mathcal{W}$ and $\mathcal{W}_*$ induce similar distributions, i.e. paraphrase. In effect, context is added to both $w_{x}$ and $w_{x^*}$ until their contextualised cases $\mathcal{W}$ and $\mathcal{W}_*$ paraphrase (Fig         \ref{fig:word_transformation1}). Note  $\mathcal{W}$ and $\mathcal{W}_*$ may have no intuitive meaning and need not correspond to a single word, unlike D\ref{def:paraphrase} paraphrases.
Alternatively, such a paraphrase can be interpreted as a transformation from $w_x\!\in\!\mathcal{W}$ to $w_{x^*}\!\in\!\mathcal{W}^*$ by adding $\winW^+$ and \textit{subtracting} $\winW^-$. 
``Subtraction'' is effected by \textit{adding words to the other side}, i.e. to $w_{x^*}
$.\footnote{Analogous to standard algebra: if $x\!<\!y$, equality is achieved either by adding to $x$ or by subtracting from $y$.}
Just as adding words to $w_x$ adds or \textit{narrows} its context, subtracting words removes or \textit{broadens} context. 
Context is thus added and removed to transform from $w_{x}$ to $w_{x^*}$, in which the paraphrase between $\mathcal{W}$ and $\mathcal{W}_*$ effectively serves as an intermediate step (Fig \ref{fig:word_transformation2}). We refer to $\mathcal{W}^+$, $\mathcal{W}^-$ as \textit{transformation parameters}, which can be thought of as \textit{explaining the difference} between $w_x$ and $w_{x^*}$ with a ``richer dictionary'' than that available to D\ref{def:paraphrase} paraphrases by including \textit{differences} between words. More precisely, transformation parameters align the induced distributions to create a paraphrase.

\begin{figure}[tb]
    \centering
    \subfloat[Adding context to each of $w_x$ and $w_{x^*}$ to reach a paraphrase.\label{fig:word_transformation1}]{%
    \centering
    {
    \begin{tikzpicture}
         \node (A) [] at (0, 0, 0) {$\mathcal{W}\approx_\textup{P}\mathcal{W}_*$};
         \node (B) [] at (-2.5, -1, 0) {$w_x$};
         \node (C) [] at ( 2.5, -1, 0) {$w_{x^*}$};
         \draw[->, >=stealth] (B) -- node [above left] {$\bm{+}\mathcal{W}^+$}(A.west);
         \draw[->, >=stealth] (C) -- node [above right] {$\bm{+}\mathcal{W}^-$}(A.east);
     \end{tikzpicture}
     }
    }%
    \hfill
    \subfloat[Adding and subtracting context to \textit{transform} $w_x$ to $w_{x^*}$.\label{fig:word_transformation2}]{%
     \centering
     {
     \begin{tikzpicture}
         \node (A) [] at (0, 0, 0) {$\mathcal{W}\approx_\textup{P}\mathcal{W}_*$};
         \node (B) [] at (-2.5, -1, 0) {$w_x$};
         \node (C) [] at ( 2.5, -1, 0) {$w_{x^*}$};
         \draw[->, >=stealth] (B) -- node [above left] {$\bm{+}\mathcal{W}^+$}(A.west);
         \draw[->, >=stealth] (A.east) -- node [above right] {$\bm{-}\mathcal{W}^-$}(C);
         \draw[->, >=stealth, red, dashed] (B) -- node [below, red] {\textit{word transformation}} (C)
        ;
    \end{tikzpicture}
    }
    }
    \caption{\small Perspectives of the paraphrase            	
	\label{fig:word_transformation}
$\mathcal{W}\approx_\textup{P}\mathcal{W}_*$.}
\end{figure}
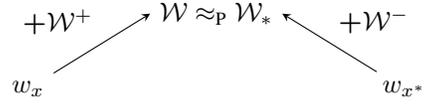
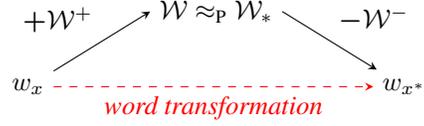
This interpretation show equivalence between a paraphrase  $\mathcal{W} \!\!\approx_\textup{P}\!\! \mathcal{W}_*$ and a word transformation -- a relationship between $w_x\!\in\!\mathcal{W}$ and $w_{x^*}\!\!\in\!\mathcal{W}_*$ based on the addition and subtraction of context that is mirrored in the addition and subtraction of embeddings.
Mathematical equivalence of the perspectives is reinforced by an alternate proof of Cor \ref{cor:word_transition} in Appendix \ref{app:cor_proof} that begins with terms in only $w_x$ and $w_{x^*\!}$, highlighting that \textit{any} words $\mathcal{W}^+$, $\mathcal{W}^-$ can be introduced, but only certain choices form the necessary paraphrase.
%
\begin{definition}\label{def:word_transformation}
	There exists a \underline{word transformation} from $w_x\!\in\!\mathcal{E}$ to $w_{x^*}\!\in\!\mathcal{E}$  with \underline{transformation parameters} $\mathcal{W}^+$, $\mathcal{W}^-\subseteq\mathcal{E}$ \emph{iff} $\{w_x\}\cup\mathcal{W}^+\!\approx_\textup{P}\!\{w_{x^*}\}\cup\mathcal{W}^-$.
\end{definition}
Note that transformation parameters may not be unique and always (trivially) include $\mathcal{W}^+\!=\!\{w_{x^*}\!\}$, $\mathcal{W}^-\!=\!\{w_{x}\}$.

\begin{figure*}[!htb]
    \centering

    \begin{minipage}{1\textwidth}
        \centering
            \begin{tikzpicture}[>=stealth, node distance=4.25cm, on grid, auto ]
                 \node (A) [text width=4.25cm,style={align=center}]{\begin{tabular}{c} ``\emph{$w_a$ is to $w_{a^*}$} \\ \emph{as} \\ \emph{$w_b$ is to $w_{b^*}$}''\end{tabular}};
                 \node (B) [right of=A, text width=2.5cm,style={align=center}]{\begin{tabular}{c} 
                 $w_a  \underset{{\msc{W}^-}}{\overset{\msc{W}^+}{\longrightarrow}} w_{a^*}$ 
                 \\ $\wedge$ \\ 
                 $w_b \underset{{\msc{W}^-}}{\overset{\msc{W}^+}{\longrightarrow}} w_{b^*}$ \end{tabular}
            };
                 \node (C) [right of=B, text width=4.25cm,style={align=center}]{\begin{tabular}{c} 
                 $\{w_a,\mathcal{W}^+\!\}\!\approx_\textup{P}\!\{w_{a^*},\mathcal{W}^-\!\}$ 
                 \\ $\wedge$ \\ 
                 $\{w_b,\mathcal{W}^+\!\}\!\approx_\textup{P}\!\{w_{b^*},\mathcal{W}^-\!\}$ \end{tabular}
            };
                 \node (D) [right of=C, text width=2.5cm,style={align=center}]{
                 \begin{tabular}{c} 
                 $\mathbf{w}_{a^*} - \mathbf{w}_{a}$ \\ $\approx$ \\ $ \mathbf{w}_{b^*} - \mathbf{w}_{b}$
                 \end{tabular}
            };
                \draw[implies-implies,double equal sign distance] (A) -- (B);
                \draw[implies-implies,double equal sign distance] (B) -- (C);
                \draw[-implies,double equal sign distance] (C) -- (D)
                ;
        \end{tikzpicture}
        \vspace{-5pt}
        \caption{\small Summary of steps to prove the relationship between analogies and word embeddings (omitting dependence error).\\$w_x \underset{{\msc{W}^-}}{\overset{\msc{W}^+}{\longrightarrow}} w_{x^*}$ denotes a word transformation $w_x$ to $w_{x^*}$ with parameters $\mathcal{W}^+, \mathcal{W}^-\!\subseteq\!\mathcal{E}$.}
    	\label{fig:summary}
    \end{minipage}
    \vspace{-5pt}
\end{figure*}
\subsection{Interpreting ``\emph{a is to a* as b is to b*}''}

With word transformation as a means of describing semantic difference between words, we mathematically interpret analogies. Specifically, we consider ``\emph{$w_x$ is to $w_{x^*}$}'' to refer to a transformation from $w_x$ to $w_{x^*}$ and an analogy to require an equivalence between such word transformations.

\begin{definition}\label{def:analogies}
	We say \emph{``$w_a$ is to $w_{a^*}$ as $w_b$ is to $w_{b^*}\!$''} for $w_a, w_b, w_{a^*}\!, w_{b^*}\!\!\in\!\mathcal{E}$ \emph{iff} there exist parameters $\mathcal{W}^+\!, \mathcal{W}^-\!\subseteq\!\mathcal{E}$ that simultaneously transform $w_a$ to $w_{a^*}$ and $w_b$ to $w_{b^*}$.
\end{definition}
%
We show that the linear relationships between word embeddings of analogies (\ref{eq:anlgy_expression}, \ref{eq:anlgy_expression_gen}) follow from D\ref{def:analogies}.

\begin{lemma}
\label{lem:analogies}
If \emph{``$w_a$ is to $w_{a^*}$ as $w_b$ is to $w_{b^*}$''} by D\ref{def:analogies} with transformation parameters $\mathcal{W}^+, \mathcal{W}^-\!\subseteq\!\mathcal{E}$, then:
\vspace{-5pt}
    \begin{align}\label{eq:lem_analogies}
        \textup{PMI}_{b^*\!}
        =  
          \textup{PMI}_{a^*\!}
        & - \textup{PMI}_{a}
        + \textup{PMI}_{b}
        \nonumber\\
            &  + 
            \bmu{\rho}^{\msc{W}^b, \msc{W}^b_*}
            -
            \bmu{\rho}^{\msc{W}^a, \msc{W}^a_*}
 			\nonumber\\
            &  + 
            (\bmu{\sigma}^{\msc{W}^b}
             \sm\,
            \bmu{\sigma}^{\msc{W}^b_*})		
              - 
            (\bmu{\sigma}^{\msc{W}^a}
             \sm\, 
            \bmu{\sigma}^{\msc{W}^a_*})		
            \nonumber\\
            &  - 
            ((\tau^\msc{W}^b
             \sm\,
            \tau^{\msc{W}^b_*})
            -
            (\tau^\msc{W}^a
             \sm\,
            \tau^{\msc{W}^a_*}))\bmu{1}
            ,
    \end{align}
where 
 $\mathcal{W}^x\!=\!\{w_x\}\cup\mathcal{W}^+$\!,\, $\mathcal{W}^x_*\!=\!\{w_{x^*}\!\}\cup\mathcal{W}^-$ for $x \!\in\! \{a, b\}$ and  $\bmu{\rho}^{\msc{W}^b, \msc{W}^b_*},            \bmu{\rho}^{\msc{W}^a, \msc{W}^a_*}$ are small.
\end{lemma}
\vspace{-15pt}
\begin{proof}
Let $\mathcal{W}\!=\!\mathcal{W}^x$,  $\mathcal{W}_*\!=\!\mathcal{W}_*^x$ for $x\!\in\!\{a, b\}$ in instances of Cor \ref{cor:word_transition} and take the difference. $\mathcal{W}^x$ paraphrases $\mathcal{W}^x_*$ for $x\!\in\!\{a, b\}$ by D\ref{def:word_transformation} and D\ref{def:analogies}. 
\end{proof}
\begin{theorem}[Analogies]
\label{thm:analogies}
If \emph{``$w_a$ is to $w_{a^*}$ as $w_b$ is to $w_{b^*}$''} by  D\ref{def:analogies} with $\mathcal{W}^+, \mathcal{W}^-\!\subseteq\!\mathcal{E}$ , then:
\vspace{-5pt}
\begin{align*}
        \mathbf{w}_{b^*}
		=\  
        \mathbf{w}_{a^*}
        & - \mathbf{w}_{a}
        + \mathbf{w}_{b}
 			\nonumber\\
            &
        + \mathbf{C}^\dagger(
            \bmu{\rho}^{\msc{W}^b, \msc{W}^b_*}
            -
            \bmu{\rho}^{\msc{W}^a, \msc{W}^a_*}
 			\nonumber\\
            &  + 
            (\bmu{\sigma}^{\msc{W}^b}
             \sm\,
            \bmu{\sigma}^{\msc{W}^b_*})		
              - 
            (\bmu{\sigma}^{\msc{W}^a}
             \sm\, 
            \bmu{\sigma}^{\msc{W}^a_*})		
            \nonumber\\
            &  - 
            ((\tau^\msc{W}^b
             \sm\,
            \tau^{\msc{W}^b_*})
            -
            (\tau^\msc{W}^a
             \sm\,
            \tau^{\msc{W}^a_*}))\bmu{1}
            ) .
\end{align*}
with terms as defined in Lem~\ref{lem:analogies}.
\end{theorem}
\vspace{-5pt}
\begin{proof}
Multiply (\ref{eq:lem_analogies}) by $\mathbf{C}^\dagger$.
\end{proof}
More generally, if D\ref{def:analogies} applies for a set of ordered word pairs $S = \{(w_{x}, w_{x^*}\!)\}$, i.e.  \emph{``$w_a$ is to $w_{a^*}$ as $w_b$ is to $w_{b^*}$''} 
$\forall\, (w_{a}, w_{a^*}\!)$, $(w_{b}, w_{b^*}\!)\!\in\! S$
with transformation parameters $\mathcal{W}^+, 
\mathcal{W}^-\!\subseteq\!\mathcal{E}$,
then each set $\{w_{x^*}, \mathcal{W}^-\!\}$ must paraphrase $\{w_{x}, \mathcal{W}^+\!\}$ by D\ref{def:word_transformation}, and (\ref{eq:analogy_from_paraphrase_gen}) holds with small paraphrase error. By this and Thm~\ref{thm:analogies} we know that word embeddings  of an analogy $\mathbf{w}_{a},  \mathbf{w}_{b}, \mathbf{w}_{a^*}, \mathbf{w}_{b^*}$ satisfy linear relationships (\ref{eq:anlgy_expression}, \ref{eq:anlgy_expression_gen}),  subject to dependence error.

A few questions remain: how to find appropriate transformation parameters; and, given non-uniqueness, which to choose? Addressing these in reverse order:

\keypoint{Transformation Parameter Equivalence}

By Lem~\ref{lem:analogies}, if \emph{``$w_a$ is to $w_{a^*}$ as $w_b$ is to $w_{b^*}$''} then, subject to dependence error: 
\begin{equation}\label{eq:transform_param_equiv2}
        \textup{PMI}_{b^*\!} - \textup{PMI}_{b} 
        \approx  
        \textup{PMI}_{a^*\!} - \textup{PMI}_{a}\ .
\end{equation}
If parameters $\mathcal{W}_2^+,\mathcal{W}_2^-$ exist that (\textit{w.l.o.g.}) transform $w_a$ to $w_{a^*}$ then (\ref{eq:lem_analogies}) holds  by suitably redefining $\mathcal{W}^x$, $\mathcal{W}^x_*$, in which $\bmu{\rho}^{\msc{W}^a, \msc{W}^a_*}$ is small but nothing is known of $\bmu{\rho}^{\msc{W}^b, \msc{W}^b_*}$. Thus, subject to dependence error:
\begin{equation}\label{eq:transform_param_equiv3}
        \qquad\textup{PMI}_{b^*\!} - \textup{PMI}_{b} 
        \approx  
        \textup{PMI}_{a^*\!} - \textup{PMI}_{a} + 
        	\bmu{\rho}^{\msc{W}^b, \msc{W}^b_*}\ .
\end{equation}
By (\ref{eq:transform_param_equiv2}), (\ref{eq:transform_param_equiv3}), subject to dependence error, $\bmu{\rho}^{\msc{W}^b, \msc{W}^b_*}$ is also small and $\mathcal{W}_2^+,\mathcal{W}_2^-$ must also transform $w_b$ to $w_{b^*}$. Thus transformation parameters of any analogical pair transform all pairs and all applicable transformation parameters can be considered equivalent, up to dependence error.
\begin{corollary}\label{cor:transform_param_equiv}
For analogy $\mathfrak{A}$, if parameters $\mathcal{W}^+$, $\mathcal{W}^-\!\subseteq\!\mathcal{E}$ transform $w_x$ to $w_{x^*}$ for any $(w_x,w_{x^*})\!\in\!S_\mathfrak{A}$, then $\mathcal{W}^+$, $\mathcal{W}^-$ simultaneously transform $w_x$ to $w_{x^*}$ $\forall (w_x,w_{x^*})\!\in\!S_\mathfrak{A}$. 
\end{corollary}

\keypoint{Identifying Transformation Parameters}

To identify ``words that explain the difference between other words'' might, in general, be non-trivial. However, by Cor \ref{cor:transform_param_equiv}, transformation parameters for analogy $\mathfrak{A}$ can simply be chosen as $\mathcal{W}^+\!=\!\{w_{x^*}\}$, $\mathcal{W}^-\!=\!\{w_x\}$ for any $(w_x,w_{x^*})\!\in\!S_\mathfrak{A}$.\footnote{In the case of an analogical question \emph{``$w_a$ is to $w_{a^*}$ as $w_b$ is to ... $?$''}, there is only one choice: $\mathcal{W}^+\!=\!\{w_{a^*}\}$, $\mathcal{W}^-\!=\!\{w_a\}$.}
Making an arbitrary choice, Thm~\ref{thm:analogies} simplifies to:
\begin{corollary}\label{cor:analogies_solved}
If ``\emph{$w_a$ is to $w_{a^*}$ as $w_b$ is to $w_{b^*}$}'' then:
    \begin{align}\label{eq:analogies_final}
        \mathbf{w}_{b^*\!}
         = 
          \mathbf{w}_{a^*\!}
        - \mathbf{w}_{a}
        + \mathbf{w}_{b}
             + \mathbf{C}^\dagger(
            \bmu{\rho}^{\msc{W}, \msc{W}_*}
            &+ 
            \bmu{\sigma}^{\msc{W}}
             - 
            \bmu{\sigma}^{\msc{W}_*}
            \nonumber\\
              - &
            (\tau^\msc{W}
             \sm\,
            \tau^{\msc{W}_*})\bmu{1}
            )
             ,
    \end{align}
where 
$\mathcal{W}\!=\!\{w_b, w_{a^*\!}\}$, $\mathcal{W}_*\!=\!\{w_{b^*\!}, w_{a}\}$ and $\bmu{\rho}^{\msc{W}, \msc{W}_*\!}$ is small.
\end{corollary}
\vspace{-10pt}
\begin{proof}
Let $\mathcal{W}^+\!=\!\{w_{a^*\!}\}, \mathcal{W}^-\!=\!\{w_{a}\}$ in Thm~\ref{thm:analogies}.
\end{proof}
\vspace{-5pt}
We arrive back at (\ref{eq:analogy_from_paraphrase}) but now link directly to analogies, proving that word embeddings of analogies satisfy linear relationships (\ref{eq:anlgy_expression}) and (\ref{eq:anlgy_expression_gen}), subject to dependence error. Fig \ref{fig:summary} shows a summary of all steps to prove Cor \ref{cor:analogies_solved}. D\ref{def:analogies} also provides a mathematical interpretation of what we mean when we say ``\emph{$w_a$ is to $w_{a^*}$ as $w_b$ is to $w_{b^*}$}''.

\subsection{Example}

To demonstrate the concepts developed, we consider the canonical analogy $\mathfrak{A}^*$: \textit{``man is to king as woman is to queen''}, for which $S_{\mathfrak{A}^*} \!=\!\{(man$, $king)$, $(woman$, $queen)\}$. By D\ref{def:analogies}, there exist parameters $\mathcal{W}^+, \mathcal{W}^-\!\subseteq\!\mathcal{E}$ that simultaneously transform $man$ to $king$ and $woman$ to $queen$, which (by Cor \ref{cor:transform_param_equiv}) can be chosen to be $\mathcal{W}^+\!=\!\{queen\}$, $\mathcal{W}^-\!=\!\{woman\}$.  Thus $\mathfrak{A}^*$ implies that $\{man,\,queen\} \!\approx_\textup{P}\!\{king,\,woman\}$  and $\{woman,\, queen\} \!\approx_\textup{P}\! \{queen,\, woman\}$, the latter being trivially true. By Cor \ref{cor:word_transition}, $\mathfrak{A}^*$ therefore implies:
\begin{align*} 
    \vv{Q}\  =\  \vv{K} - \vv{M} + \vv{W} 
    + \mathbf{C}^\dagger(
            \bmu{\rho}^{\msc{W}, \msc{W}_*}
            & + 
            \bmu{\sigma}^{\msc{W}}
             - 
            \bmu{\sigma}^{\msc{W}_*}
            \nonumber\\
            &\ \   - 
            (\tau^\msc{W}
             \sm\,
            \tau^{\msc{W}_*})\bmu{1}
		)\ ,
\end{align*}
\vspace{-5pt}
where we abbreviate words by their initials and, explicitly:
\begin{align*} 
 \bmu{\rho}^{\msc{W}, \msc{W}_*\!} 
 & = 
 \log \tfrac{p(c_j|w_Q,w_M)}{p(c_j|w_W,w_K)} &  \text{(whi}&\text{ch must be small)},
 \\
\bmu{\sigma}^{\msc{W}} 
 & \!=\! 
 \log \tfrac{p(w_W,       w_K|c_j)}{p(w_W|c_j)p(w_K|c_j)}
 ,\  &
\tau^\msc{W} 
 & \!=\! 
 \log \tfrac{p(w_W , w_K)}{p(w_W)p(w_K)} ,
 \\
\bmu{\sigma}^{\msc{W}_*} 
  &\!=\!  
 \log \tfrac{p(w_Q,       w_M|c_j)}{p(w_Q|c_j)p(w_M|c_j)}
 ,\ &
\tau^\msc{W_*} 
 & \!=\! 
 \log \tfrac{p(w_Q,  w_M)}{p(w_Q)p(w_M)} \ .
\end{align*}

Thus $\vv{Q} \approx \vv{K} - \vv{M} + \vv{W}$ subject to the accuracy with which $\{man,queen\}$ paraphrases $\{king,woman\}$ and statistical dependencies within those word pairs (see Fig \ref{fig:gap_explained}).

\begin{figure}[!t]
 \centering
    \includegraphics[width=\linewidth, trim=1.25cm 1.9cm 1cm 3.0cm, clip]{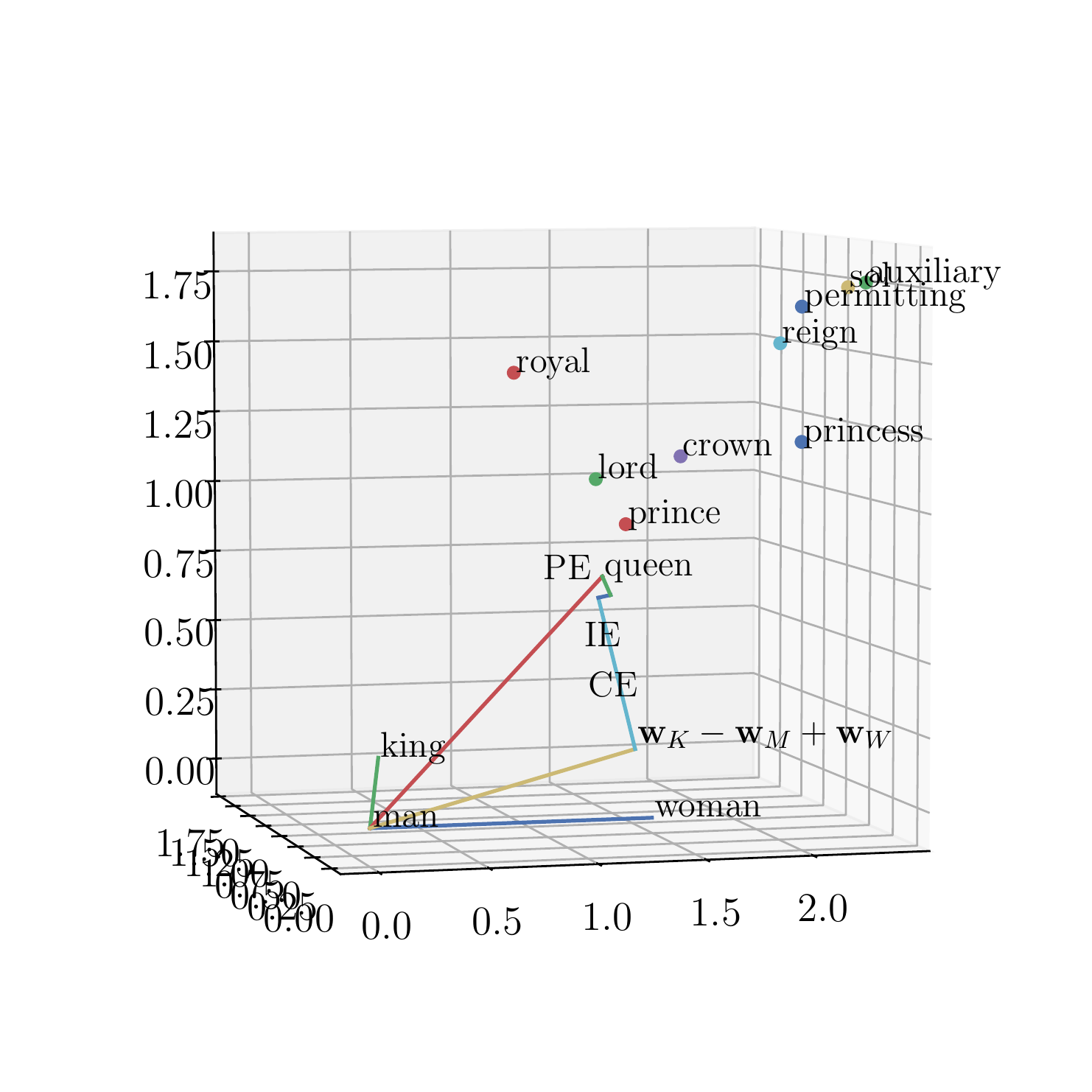}
    \caption{The plot shows the same embeddings of Fig \ref{fig:the_gap}, now with the difference between  $\vv{K}-\vv{M}+\vv{W}$ and the embedding of $queen$ explained (see connecting ``zigzag'') as the sum of conditional independence error (CE), independence error (IE) and paraphrase error (PE). As anticipated, their sum is smallest for \textit{queen}. Related words are seen nearby, with unrelated words clustered further away. 
    Plot generated by fixing the $xy$ plane to contain $man$, $king$, $queen$ and all other vectors plotted relatively, i.e. the $z$-axis captures any component off the $xy$-plane. Values are computed from the ``text8'' corpus \cite{text8}.}
    \label{fig:gap_explained}
    \vspace{-10pt}
\end{figure}

\subsection{Dependence error in analogies}\label{sec:similar_dep}

Dependence error terms for analogies (\ref{eq:lem_analogies}) bear an important distinction from those in one-word paraphrases (\ref{eq:lem1}). When a word set $\mathcal{W}$ is paraphrased by a single word $w_*$, the dependence error comprises a conditional independence term ($\bmu{\sigma}^\msc{W}$) and a mutual independence term ($\tau^\msc{W}\bmu{1}$) that bear no obvious relationship to one another and can only cancel by chance, which is low in high dimensions. However, (\ref{eq:lem_analogies}) contains offsetting pairs of each component ($\bmu{\sigma}^\msc{W}, \bmu{\sigma}^\msc{W_*}, \tau^\msc{W}, \tau^\msc{W_*}$), i.e. terms of the same form that may cancel, thus word sets with \textit{similar dependence terms} will paraphrase with small overall dependence error. 

It is illustrative to consider the case $w_a\!\!=\!\!w_b$, $w_{a^*}\!\!\!=\!\!w_{b^*}\!$, corresponding to the trivial analogy ``$w_a$ is to $w_{a^*}$ as ``$w_a$ is to $w_{a^*}$\!'', which holds true with zero total error for any word pair. Considering specific error terms: the paraphrase error is zero since $p(c_j|\{w_a, w_{a^*}\!\})\!=\!p(c_j|\{w_{a^*\!},w_a\}),\ \forall c_j\!\in\!\mathcal{E}$, thus the net dependence error is also zero. However, individual dependence error terms, e.g. $\log\tfrac{p(w_a, w_{a^*})}{p(w_a)p(w_{a^*})}$, are generally non-zero. This therefore proves existence of a case in which non-zero dependence error terms negate one another to give a negligible net dependence error.

\subsection{Analogies in \textit{explicit} embeddings} 

As with paraphrases, analogical relationships in embeddings stem from relationships between columns of $\mathbf{PMI}$.
\begin{corollary}\label{cor:analogy_explicit}
    Cor \ref{cor:analogies_solved} applies to \textit{explicit} (full-rank) embeddings, i.e. columns of $\mathbf{PMI}$, with $\mathbf{C}\!=\!\mathbf{I}$ (the identity matrix).
\end{corollary}

\subsection{Analogies in W2V embeddings}\label{sec:analogy_w2v}

As with paraphrases (Sec \ref{sec:addition_w2v}), the results for analogies can be extended to W2V embeddings by including the \textit{shift} term appropriately throughout. Since the transformation parameters for analogies are of equal size (i.e. $|\mathcal{W}^+|=|\mathcal{W}^-|=1$), we find that all \textit{shift} terms cancel.
\begin{corollary}\label{cor:thm2_w2v}
	Cor \ref{cor:analogies_solved} applies to W2V embeddings replacing the projection $\mathbf{C}^\dagger(\cdot)$ with $f_{W2V}(\cdot)$.
\end{corollary}
Thus, linear relationships between embeddings for analogies hold equally for W2V embeddings as for those derived without the \textit{shift} distortion. Whilst perhaps surprising, this is corroborative since linear analogical relationships have been observed extensively in W2V embeddings (e.g. \citet{levy2014linguistic}), as is now justified theoretically. Thus we know that analogies hold for W2V embeddings subject to higher order statistical relationships between words of the analogy as defined by the paraphrase and dependence errors.

\section{Conclusion}

In this work, we develop a probabilistically principled definition of \textit{paraphrasing} by which equivalence is drawn between words and word sets by reference to the distributions they induce over words around them. 
We prove that, subject to statistical dependencies, paraphrase relationships give rise to linear relationships between word embeddings that factorise PMI (including columns of the PMI matrix), and thus others that approximate such a factorisation, e.g. W2V and \textit{Glove}. 
By showing that paraphrases can be interpreted as \textit{word transformations}, we enable analogies to be mathematically defined and, thereby, properties of semantics to be translated into properties of word embeddings. This provides the first rigorous explanation for the presence of linear relationships between the word embeddings of analogies. 


In future work we aim to extend our understanding of the relationships between word embeddings to other applications of discrete object representation that rely on an underlying matrix factorisation, e.g. graph embeddings and recommender systems. Also, word embeddings are known to capture stereotypes present in corpora (\citet{bolukbasi2016man}) and future work may look at developing our understanding of embedding composition to foster principled methods to correct or \textit{debias} embeddings.

\section*{Acknowledgements}

We thank Ivana Bala\v{z}evi\'c and Jonathan Mallinson for helpful comments on this manuscript. Carl Allen was supported by the Centre for Doctoral Training in Data Science, funded by EPSRC (grant EP/L016427/1) and the University of Edinburgh.

\bibliography{example_paper}
\bibliographystyle{icml2019}

\appendix

\newpage
\section*{Appendices}

\section{The KL-divergence between induced distributions}\label{app:KL}
We consider the words found by minimising the difference KL-divergences considered in Section \ref{sec:paraphrases}. Specifically:
\begin{align*}
    w_*^{\mathsmaller{(1)}} 
    & = \argmin_{w_i\in\mathcal{E}}
    D_\mathsmaller{KL}[\,
    p(c_j| \mathcal{W})
     \,||\,
    p(c_j| w_i)
    \,]
    \\
    w_*^{\mathsmaller{(2)}} 
    & = \argmin_{w_i\in\mathcal{E}}
    D_\mathsmaller{KL}[\,
    p(c_j| w_i)
     \,||\,
    p(c_j| \mathcal{W})
    \,]
\end{align*}

Minimising $D_\mathsmaller{KL}[\,
    p(c_j| \mathcal{W})
     \,||\,
    p(c_j| w_i)
    \,]$ identifies the word that induces a probability distribution over context words closest to that induced by $\mathcal{W}$, in which probability mass is assigned to $\smash{c_j}$ \textit{wherever} it is for $\mathcal{W}$. Intuitively, $w_*^{\mathsmaller{(1)}}$ is the word that most closely reflects \textit{all} aspects of $\mathcal{W}$, and may occur in contexts where no word $\winW$ does. 

Minimising $D_\mathsmaller{KL}[\,
    p(c_j| w_i)
     \,||\,
    p(c_j| \mathcal{W})
    \,]$  finds the word that induces a distribution over context words that is closest to that induced by $\mathcal{W}$, in which probability mass is assigned as broadly as possible but \textit{only} to those $c_j$ to which probability mass is assigned for $\mathcal{W}$. Intuitively, $w_*^{\mathsmaller{(2)}}$ is the word that reflects as many aspects of $\mathcal{W}$ as possible, as closely as possible, but nothing additional, e.g. by having other meaning that $\mathcal{W}$ does not.

\subsection{Weakening the paraphrase assumption}

    For a given word set $\mathcal{W}$, we consider the relationship between embedding sum $\vv{\msc{W}}$ and embedding $\vv{*}$ for the word $w_*\!\in\!\mathcal{E}$ that minimises the KL-divergence (we illustrate with $\Delta_{\mathsmaller{KL}}^{\msc{W},w_*}$).
    Exploring a weaker assumption than D\ref{def:paraphrase}, tests whether D\ref{def:paraphrase} might exceed requirement, and explores the relationship between $\vv{*}$ and $\vv{\msc{W}}$ as paraphrase error increases.
    \begin{theorem}[Weak paraphrasing]\label{thm:paraphrase_KL}
        For $w_*\!\in\!\mathcal{E}, \mathcal{W}\!\subseteq\!\mathcal{E}$, \emph{if} $w_*$ minimises $\Delta_{\mathsmaller{KL}}^{\msc{W},w_*} \!\doteq\! D_\mathsmaller{KL}[\,p(c_j|\mathcal{W})\,||\,p(c_j|w_*)\,]$,
        then:
        \begin{equation}
            {\vv{*}}^\top\mathbf{\hat{c}}
            \ =\ 
            {{\vv{\msc{W}}}}^{\top}{\mathbf{\hat{c}}}
            - \Delta_{\mathsmaller{KL}}^{\msc{W},w_*}
            + \hat{\sigma}^{\msc{W}}
            - \tau^\msc{W}
        \end{equation}
        where $\mathbf{\hat{c}}=\! \mathbb{E}_{j|\msc{W}}[\mathbf{c}_{j}]$, $\hat{\sigma}^{\msc{W}}=\! \mathbb{E}_{j|\msc{W}}[\bmu{\sigma}^{\msc{W}}_j]$ and $\mathbb{E}_{j|\msc{W}}[\cdot]$ denotes expectation under $p(c_j|\mathcal{W})$.
    \end{theorem}
    \begin{proof}
        \begin{align*}
            \Delta_{\mathsmaller{KL}}^{\msc{W},w_*}
            & = 
            \mathsmaller{\sum_j} p(c_j| \mathcal{W})\log\tfrac{p(c_j| \mathcal{W})}{p(c_j| w_*)}
            \nonumber\\
            & \overset{(\ref{eq:lem1})}{=} 
            \mathbb{E}_{j|\msc{W}}[\, \mathsmaller{\sum_i}\text{PMI}(w_i,c_j)
            \nonumber\\ & \qquad\quad \ 
            -\text{PMI}(w_*,c_j) 
            + \bmu{\sigma}^{\msc{W}}_j
            - \tau^\msc{W}]
            \nonumber\\ 
            & = 
            \mathbb{E}_{j|\msc{W}}[{\vv{\msc{W}}}^\top\mathbf{c}_{j}
            - {\vv{*}}^\top\mathbf{c}_{j}] 
            + \hat{\sigma}^{\msc{W}}
            - \tau^\msc{W}\qedhere
        \end{align*} 
    \end{proof}
    Thus, the weaker paraphrase relationship specifies a hyperplane containing $\vv{*}$ and so does not uniquely define $\vv{*}$ (as under D\ref{def:paraphrase}) and cannot explain the observation of embedding addition for paraphrases (as suggested by \citet{gittens2017skip}). A similar result holds for $\Delta_{\mathsmaller{KL}}^{w_*,\msc{W}}$.
    In principle, Thm~\ref{thm:paraphrase_KL} could help locate embeddings of words that more loosely paraphrase $\mathcal{W}$, i.e. with increased paraphrase error.

\section{Proof of Lemma \ref{lem:paraphrase}}
\label{app:paraphrase_alt_proof}
\primelemma*
\begin{proof}
    \begin{align*}
        \text{PMI}(w_*, c_j)
         & - \sum_{\winW} \text{PMI}(w_i,\!c_j)
        \nonumber\\
        & =
          \log\frac{p(w_*|c_j)}{p(w_*)}
        - \log\prod_{\winW} \frac{p(w_i|c_j)}{p(w_i)}
        \nonumber\\
        & =
          \log \frac{p(w_*|c_j)}{\prod_{\msc{W}} p(w_i|c_j)}
        - \log \frac{p(w_*)}{\prod_{\msc{W}} p(w_i)}
        \nonumber\\
        & \qquad \qquad 
          \color{blue}
        + \log \frac{p(\mathcal{W}|c_j)}{p(\mathcal{W}|c_j)}
          \color{red}
        + \log \frac{p(\mathcal{W})}{p(\mathcal{W})}
          \color{black}
        \nonumber\\
       & =
          \log \frac{p(w_*|c_j)}
          {\color{blue}p(\mathcal{W}|c_j)\color{black}}
        - \log \frac{p(w_*)}
        {\color{red}p(\mathcal{W})\color{black}}
        \nonumber\\
        & \qquad 
        + \log \frac{\color{blue}p(\mathcal{W}|c_j)\color{black}}
        {\prod_{\msc{W}} p(w_i|c_j)}
        - \log \frac{\color{red}p(\mathcal{W})\color{black}}
        {\prod_{\msc{W}} p(w_i)}
        \nonumber\\
       & =
          \log \frac{p(c_j|w_*)}
          {\color{purple}p(c_j|\mathcal{W})\color{black}}
        + \log \frac{\color{blue}p(\mathcal{W}|c_j)\color{black}}
        {\prod_{\msc{W}} p(w_i|c_j)}
        \nonumber
        \\
        & \qquad \qquad \qquad \qquad \qquad\ \  
        - \log \frac{\color{red}p(\mathcal{W})\color{black}}
        {\prod_{\msc{W}} p(w_i)}
        \nonumber\\
        &  =
         \bmu{\rho}^{\msc{W}, w_*}_j  
         \ + \ \bmu{\sigma}^{\msc{W}}_j
         \ - \ \tau^\msc{W}\ ,
    \end{align*}
where, unless stated explicitly, products are with respect to all $w_i$ in the set indicated.
\end{proof}
Introduced terms are highlighted to show their evolution within the proof. At the step where terms are introduced, the existing error terms have no statistical meaning. This is resolved by introducing terms to which both error terms can be meaningfully related, through paraphrasing and independence.

\newpage
\section{Proof of Lemma \ref{lem:paraphrase_extended}}
\label{app:paraphrase_ext_proof}
\primelemmatwo*
\begin{proof}
    \begin{align*}
        \sum_{w_i\in\msc{W}_*} \text{PMI}&(w_i,\!c_j)
         - \sum_{w_i\in\msc{W}} \text{PMI}(w_i,\!c_j)
        \nonumber\\
        & =
          \log\!\prod_{w_i\in\msc{W}_*}\! \frac{p(w_i|c_j)}{p(w_i)}
        - \log\!\prod_{w_i\in\msc{W}}\! \frac{p(w_i|c_j)}{p(w_i)}
        \nonumber\\
        & =
          \log \frac{\mathsmaller{\prod_\msc{W}_*} p(w_i|c_j)}
          {\mathsmaller{\prod_\msc{W\phantom{_*}}} p(w_i|c_j)}
        - \log \frac{\mathsmaller{\prod_\msc{W}_*} p(w_i)}
          {\mathsmaller{\prod_\msc{W\phantom{_*}}} p(w_i)}
        \nonumber\\
        & \qquad \quad 
          \color{green}
        + \log \frac{p(\mathcal{W}_*|c_j)}{p(\mathcal{W}_*|c_j)}
          \color{orange}
        + \log \frac{p(\mathcal{W}_*)}{p(\mathcal{W}_*)}
          \color{black}
        \nonumber\\
        & \qquad \quad 
          \color{blue}
        + \log \frac{p(\mathcal{W}|c_j)}{p(\mathcal{W}|c_j)}
          \color{red}
        + \log \frac{p(\mathcal{W})}{p(\mathcal{W})}
        \nonumber\\
        & =
        + \log \frac
        {\color{green}p(\mathcal{W}_*|c_j)\color{black}}
        {\color{blue}p(\mathcal{W}|c_j)\color{black}}
        - \log \frac
        {\color{orange}p(\mathcal{W}_*)\color{black}}
        {\color{red}p(\mathcal{W})\color{black}}
        \nonumber\\
        & \qquad 
        + \log \frac
          {\mathsmaller{\prod_\msc{W}_*} p(w_i|c_j)}
          {\color{green}p(\mathcal{W}_*|c_j)\color{black}}
        - \log \frac
          {\mathsmaller{\prod_\msc{W}_*} p(w_i)}
          {\color{orange}p(\mathcal{W}_*)\color{black}}
        \nonumber\\
        & \qquad 
        + \log \frac
        {\color{blue}p(\mathcal{W}|c_j)\color{black}}
        {\mathsmaller{\prod_\msc{W}} p(w_i|c_j)}
        - \log \frac
        {\color{red}p(\mathcal{W})\color{black}}
        {\mathsmaller{\prod_\msc{W}} p(w_i)}
        \nonumber\\
        & =
        + \log \frac
        {\color{brown}p(c_j|\mathcal{W}_*)\color{black}}
        {\color{purple} p(c_j|\mathcal{W}  )\color{black}}
        \nonumber\\
        & \qquad 
        + \log \frac
        {\color{blue}p(\mathcal{W}|c_j)\color{black}}
        {\mathsmaller{\prod_\msc{W}} p(w_i|c_j)}
        - \log \frac
          {\color{green}p(\mathcal{W}_*|c_j)\color{black}}
          {\mathsmaller{\prod_\msc{W}_*} p(w_i|c_j)}
        \nonumber\\
        & \qquad 
        - \log \frac
        {\color{red}p(\mathcal{W})\color{black}}
        {\mathsmaller{\prod_\msc{W}} p(w_i)}
        + \log \frac
          {\color{orange}p(\mathcal{W}_*)\color{black}}
          {\mathsmaller{\prod_\msc{W}_*} p(w_i)}
        \nonumber\\
        &  =
            \bmu{\rho}^{\msc{W}, \msc{W}_*}_j
             + 
            \bmu{\sigma}^{\msc{W}}_j
             - 
            \bmu{\sigma}^{\msc{W}_*}_j
             - 
            (\tau^\msc{W}
             -
            \tau^{\msc{W}_*})\ ,
    \end{align*}
where, unless stated explicitly, products are with respect to all $w_i$ in the set indicated.
\end{proof}
The proof is analogous to that of Lem~\ref{lem:paraphrase}, with more terms added (as highlighted) to an equivalent effect. A key difference to single-word (or \textit{direct}) paraphrases (D\ref{def:paraphrase}) is that the paraphrase is between two word sets $\mathcal{W}$ and $\mathcal{W}_*$ that need not correspond to any single word. The paraphrase error $\bmu{\rho}^{\msc{W}, \msc{W}_*}$  compares the induced distributions of the two sets, following the same principles as direct paraphrasing, but with perhaps less interpretatability.

\newpage
\section{Alternate Proof of Corollary \ref{cor:word_transition}}
\label{app:cor_proof}
\primecorollary*
\begin{proof}
\begin{align*}
	\text{PMI}(w_{x^*},&c_j) - \text{PMI}(w_x,c_j)
	\\
	& =
	\log\frac{p(c_j|w_{x^*\!})}{p(c_j|w_x)}
		\ + \log \prod_{w_i\in\msc{W}^+\!}
		\frac{p(c_j|w_i)}{p(c_j|w_i)}
    \\
	& \qquad\qquad\qquad\qquad\quad
	+ \log \prod_{w_i\in\msc{W}^-\!}\frac{p(c_j|w_i)}
		             		{p(c_j|w_i)}  	
	\\
	& = 
	\sum_{w_i\in\msc{W}^+\!}\log p(c_j|w_i)
	 \ - \sum_{w_i\in\msc{W}^-\!}\log p(c_j|w_i)  
	 \\
	& \qquad\qquad\qquad\qquad\quad
		+ \log \frac{\mathsmaller{\prod_{\msc{W}_*}}p(c_j|w_i)}
		            {\mathsmaller{\prod_{\msc{W}\phantom{_*}}}p(c_j|w_i)}
  	\\
	& =
	   \sum_{w_i\in\msc{W}^+\!} \text{PMI}(w_i,c_j)
	 \ - \sum_{w_i\in\msc{W}^-\!} \text{PMI}(w_i,c_j)
	 \\
	& \qquad\qquad\qquad\quad
		+ \log 
		\frac{
		\mathsmaller{\prod_{\msc{W}_*}p(w_i|c_j)\, \prod_{\msc{W\phantom{_*}}}p(w_i)}}
		{\mathsmaller{\prod_{\msc{W\phantom{_*}}}p(w_i|c_j)\, \prod_{\msc{W}_*}p(w_i)}}
  	\\
	& =
	   \sum_{w_i\in\msc{W}^+\!} \text{PMI}(w_i,c_j)
	 \ - \sum_{w_i\in\msc{W}^-\!} \text{PMI}(w_i,c_j)
	\\
	& \qquad\quad\ \ 
		+ \log \frac{p(c_j|w_{x^*\!},W^-)}
		             {p(c_j|w_x,\phantom{^*\!}W^+)} 					        
	\\
	& \qquad\quad\ \ 
		+ \log \frac{\mathsmaller{\prod_{\msc{W}_*}}p(w_i|c_j)}
		             {p(w_{x^*},W^-|c_j)}
		       \frac{p(w_x    ,W^+|c_j)} 
		             {\mathsmaller{\prod_{\msc{W}}}p(w_i|c_j)}
	\\
	& \qquad\quad\ \ 
		- \log \frac{\mathsmaller{\prod_{\msc{W}_*}}p(w_i)}
		             {p(w_{x^*},W^-)}
		       \frac{p(w_x    ,W^+)} 
		             {\mathsmaller{\prod_{\msc{W}}}p(w_i)}
  	\\
	& =
	   \sum_{w_i\in\msc{W}^+\!} \text{PMI}(w_i,c_j)
	 \ - \sum_{w_i\in\msc{W}^-\!} \text{PMI}(w_i,c_j)
	 \\
	& \qquad\ \ \ 
	  	+ \bmu{\rho}^{\msc{W},\msc{W}_*\!}_j  		 	  	\ 		+ \bmu{\sigma}^{\msc{W}}_j
	    - \bmu{\sigma}^{\msc{W}_*}_j
	 - (\tau^{\msc{W}}
	 -  \tau^{\msc{W}_*})\, ,
\end{align*}
where, unless stated explicitly, products are with respect to all $w_i$ in the set indicated; and $\mathcal{W}\!=\!\{w_x\}\cup\mathcal{W}^+$, $\mathcal{W}_*\!=\!\{w_{x^*}\}\cup\mathcal{W}^-$ to lighten notation.
Multiplying by $\mathbf{C}^\dagger$ completes the proof.
\end{proof}






\end{document}